\documentclass[letterpaper, 10 pt, conference]{ieeeconf/ieeeconf}

\IEEEoverridecommandlockouts 

\usepackage[utf8]{inputenc}
\let \labelindent\relax

\usepackage{amsmath}
\usepackage{amssymb}
\usepackage{epsfig, psfrag}
\usepackage{amsthm}
\usepackage{latexsym}
\usepackage{bbm}
\usepackage{soul}
\usepackage{mathrsfs}
\usepackage{algorithm}
\usepackage{multicol}
\usepackage[noend]{algpseudocode}
\usepackage{algorithmicx}

\usepackage{amsfonts} 
\usepackage{dsfont} 
\usepackage{euscript} 
\usepackage{amsthm}
\usepackage{enumitem}
\usepackage{mathtools}

\algnewcommand\algorithmicinput{\textbf{INPUT:}}
\algnewcommand\INPUT{\item[\algorithmicinput]}

\allowdisplaybreaks
\usepackage{mathtools}

\usepackage{epstopdf}
\usepackage{subfigure}
\usepackage{graphics,color}
\usepackage{graphicx}
\usepackage{xcolor}
\usepackage{lipsum}
\usepackage{soul}
 
\newtheorem{theorem}{Theorem}
\newtheorem{lemma}{Lemma}

\newtheorem{assumption}{Assumption}

\theoremstyle{definition}
\newtheorem{definition} {Definition}
\newtheorem{remarks}{Remark}

\newcounter{relctr} 
\everydisplay\expandafter{\the\everydisplay\setcounter{relctr}{0}} 

\AtBeginDocument{}

\def\argmax{\mathop{\rm arg\,max}}

\usepackage{titlesec}

\usepackage[noadjust]{cite}

\makeatletter

\title{Provably Efficient Multi-Agent Reinforcement Learning \\ with Fully Decentralized Communication}

\author{Justin Lidard, Udari Madhushani, and Naomi Ehrich Leonard
\thanks{This research has been supported in part by an NDSEG Fellowship and ONR grants N00014-18-1-2873 and N00014-19-1-2556.}
\thanks{Authors are with Department of Mechanical and Aerospace Engineering, Princeton University, Princeton, NJ 08544, USA.
        {\tt\small \{jlidard,udarim, naomi\}@princeton.edu}}%
}

\date{March 2021}

\begin{document}

\maketitle
\thispagestyle{plain}
\pagestyle{plain}

\begin{abstract}
    A challenge in reinforcement learning (RL) is minimizing the cost of sampling associated with exploration. Distributed exploration reduces sampling complexity in multi-agent RL (MARL). We investigate the benefits to performance in MARL when exploration is fully decentralized. Specifically, we consider a class of online, episodic,  tabular $Q$-learning problems under time-varying reward and transition dynamics, in which agents can communicate in a decentralized manner. We show that group performance, as measured by the bound on regret, can be significantly improved through communication when each agent uses a decentralized message-passing protocol, even when limited to sending information up to its $\gamma$-hop neighbors. We prove regret and sample complexity bounds that depend on the number of agents, communication network structure and $\gamma$. We show that incorporating more agents and more information sharing into the group learning scheme speeds up convergence to the optimal policy. Numerical simulations illustrate our results and validate
our theoretical claims.  
\end{abstract}




\section{Introduction and Related Work}

Multi-agent reinforcement learning is an active and growing area of research with focus on cooperation, competition and mixed-motive scenarios that arise in multi-agent systems. Similar to the single-agent setting, in MARL agents try to maximize their cumulative reward through estimation of a \textit{value function}. We measure the convergence of this estimate in two ways: {\em sample complexity}, which bounds the number of reward samples needed to find an approximately optimal value function, and {\em regret}, which bounds the cumulative value function error over time. We denote an algorithm as \textit{sample efficient} if it has near-optimal sample complexity.

In training MARL algorithms, allowing agents to share value function parameters or reward samples can lead to faster convergence. But, it still remains an open question whether this information sharing should be centralized or decentralized. There has been empirical evidence \cite{rashid_weighted_2020, foerster_counterfactual_2018, wang_roma_2020-1} that training using a central controller is effective even when the state and action spaces are large. However, as the number of agents increases, the number of joint states and actions becomes exponentially large, demanding more storage space to tabulate scenarios, increasing the \textit{space complexity}. Due to the distributed nature and  combinatorial complexity of MARL, there has been increased interest in developing theory for \textit{decentralized} model-free algorithms that allow agents to train and to perform optimally with space complexity remaining polynomial in the number of agents \cite{dubey_provably_2021, zhang_finite-sample_2021-1, zhang_marl_2021}, particularly when the problem of optimizing the joint action factors as optimizing individual agent actions.  

A key underlying trade-off in RL is known as \textit{exploration} versus \textit{exploitation}; an efficient exploration strategy is always necessary to discover new scenarios while capitalizing on experience from prior scenarios. Tabular upper-confidence bound (UCB) algorithms, such as those utilized in episodic model-free $Q$-learning \cite{jin_is_2018, zhang_almost_2020-1}, choose a time-varying exploration bonus leading to $\Tilde{\mathcal O}(\sqrt{H^4SAT})$ (Hoeffding-style) and $\Tilde{\mathcal O}(\sqrt{H^3SAT})$ (Bernstein-style) regret\footnote{$\Tilde{\mathcal O} $ ignores $\log$ terms.}. The authors of \cite{zhang_almost_2020-1} improve this regret bound to $\mathcal O(\sqrt{H^2SAT})$, which is proved in \cite{jin_is_2018, jin_bellman_2021} to be minimax-optimal in the single-agent case. Here, $S$ is the number of states, $A$ is the number of actions, $H$ is the number of steps per episode, and $T$ is the total number of reward samples. The exploration bonus is chosen to match the value function error up to a constant factor. In the multi-agent setting, a naïve application of the single-agent result \cite{jin_is_2018} running in parallel gives $\tilde{\mathcal O}(M\sqrt{H^4SAT})$ regret, where $M$ is the number of agents. In this paper, we show how $Q$-learning can give $\Tilde{\mathcal O}(\sqrt{MH^4SAT})$ regret by providing an optimal exploration strategy that uses communication among the agents to accelerate online learning.  


Our results apply to tasks in which each agent in a network learns an optimal value function under the episodic Markov decision process (MDP) paradigm. We focus on the scenario where agents interact with their respective MDP in parallel, and there is no coupling between joint actions and the joint reward. An agent is allowed to perform one action and one round of message passing per time step up to its $\gamma$-hop neighbors; this approach is analogous to consensus protocols in earlier works (e.g. \cite{kar_qd-learning_2013, zhang_fully_2018-1, zhang_networked_2018, zhang_marl_2021}) but permits sharing of state, action, and reward information directly. To this end, we ask the question: if agents operate in similar environments, can they jointly explore the state-action space and, through communication, discover the optimal policy faster than individual agents operating in parallel? We propose an algorithm that permits \textit{full decentralization} of the learning process in which agents explore proportional to the amount of information they receive from other agents. 

Kar et al.  were among the first to show almost-sure convergence of consensus-based $Q$-learning over a sparse communication network \cite{kar_qd-learning_2013}. Several works introduce a decentralized actor-critic approach, also using consensus to provide convergence guarantees in the cooperative setting, although no finite-time sample complexity results are provided \cite{zhang_fully_2018-1, zhang_networked_2018}. J. Zhang et al. provide $\mathcal O(\varepsilon^{-2.5})$ (one round of consensus) and $\mathcal O(\varepsilon^{-2})$ (multiple rounds) sample complexity for a cooperative batch actor-critic algorithm \cite{zhang_marl_2021}. Our method differs by permitting an arbitrary constant number of episodes, needing only one round of communication per iteration for optimal regret, and being value-based only. K. Zhang et al. provide a finite-sample PAC bound in the cooperative and competitive batch (i.e. not online) settings \cite{zhang_finite-sample_2021-1}. Dubey and Pentland provide multi-agent regret bounds for cooperative RL for parallel MDPs (see Section \ref{preliminaries})  \cite{dubey_provably_2021}. A linear function approximation and a central server are used to perform least-squares value iteration with shared transition samples. Our work complements these results by providing the first decentralized multi-agent regret bound for tabular $Q$-learning, with message passing eliminating the need for a central server.  In the sequel, we show further that our bound matches centralized benchmarks.

The key contributions of this paper are as follows: (1) we provide a novel multi-agent UCB $Q$-learning algorithm in which agents use message passing to share information, and (2) we  show  that  group  performance, as  measured  by  the  bound  on  regret, improves upon the Hoeffding-style exploration strategy in the single agent setting by a factor of $M^{-1/2}$. Moreover, our regret takes into consideration the network structure and communication threshold $\gamma$, suggesting that even mild communication leads to improvement in the regret. As far as we know, this paper provides the first multi-agent regret bound for decentralized tabular $Q$-learning for a general network.

\section{Mathematical Preliminaries} \label{preliminaries}

We consider \textit{parallel MDPs} \cite{dubey_provably_2021, bernstein_complexity_2002}, which are defined as a collection $\{\textbf{MDP}(\mathcal S_i, \mathcal A_i, \mathbb P_i, r_i, H)\}_{i=1}^M$, where each agent $i \in [M]$ plans over a finite time horizon $H$ and has access to identical state space $\mathcal S_i = \mathcal S_j$ and action space $\mathcal A_i = \mathcal A_j$, $\forall i,j \in [M]$.\footnote{$[N]=\{1,...,N\}$ for all natural numbers $N$.} We denote the size of the state and action spaces as $S:=|\mathcal S|$ and $A:=|\mathcal A|$ respectively. 

The local reward functions $r_i$ and transitions $\mathbb P_i$ can be inhomogenous (over time) but only depend on the local state and action information of each agent. Thus, each agent interacts \textit{in parallel} with their corresponding MDP, and there is no coupling between the state and actions chosen by an agent and the reward received by any other agent. Here, we let $r = r_1 = ... = r_M$ and $\mathbb P = \mathbb P_1 = ... = \mathbb P_M$, such that every agent interacts with the same MDP. Although restrictive, parallel MDPs provide a baseline for generalizations to more complex environments, such as heterogeneous MDPs \cite{zhou_nearly_2021}.

Dubey and Pentland \cite{dubey_provably_2021} provide a multi-agent regret bound under heterogeneous reward and transition structure using estimation of the feature covariance and bias. In contrast, our framework permits direct sharing of transition-reward tuples where the $Q$-function can be tabulated directly. We let the $M$ agents communicate over a network $G$ with edge set $E$. In the episodic setting, each agent $m\in[M]$ gets a local reward $r_{h} (x_{m,h}^k,a_{m,h}^k) $, the state-action pair it chooses at step $h$ of episode $k$. We let the value function $V_{m,h}^k$  denote agent $m$'s estimate of the expected lookahead cumulative reward $V_{m,h}^*$. Agent $m$ follows a deterministic policy $\pi_{m,k}$ and seeks to maximize the true payoff $V_{m,h}^{\pi_{m,k}}$. Agent $m$'s regret at episode $K$ is defined as the cumulative difference in value functions: \[\text{Regret}_m(K) := \sum_{k=1}^K (V_{m,h}^* - V_{m,h}^{\pi_{m,k}}) (x_{m,1}^k)\]  (see \cite{jin_is_2018}). 
In the Parallel MDP construction, we assume that joint reward and transition probability are decoupled. The objective is then to maximize the joint reward through optimization of individual rewards. 
\textit{Group regret} is
\begin{equation}
    \text{Regret}_G(K) = \sum_{m=1}^M \text{Regret}_m(K) \label{eqn:group_regret}
\end{equation}
Here, we show that the local optimizations are  provably accelerated through  communication. 

\section{MARL with Full Communication} \label{regret_analysis}

\subsection{Algorithm} \label{sec:Algorithm}

We consider a multi-agent extension to the online $Q$-learning algorithm provided by \cite{jin_is_2018}. Consider $M$ agents operating in a parallel MDP for $K$ episodes of length $H$. In the single-agent case, each agent makes one update corresponding to the state-action pair $(x_h^k, a_h^k)$ visited at each step $h$ of episode $k$, for a total of $T=HK$ samples. In the multi-agent case, 
at each time $\tau$, each agent $m\in[M]$ sees a state-action pair $(x_{m,h}^k, a_{m,h}^k)$, reward $r_h(x_{m,h}^k, a_{m,h}^k)$, and next state $x_{m,h+1}^k$ and exchanges messages with its neighbors. We consider a message-passing protocol in which each agent $m\in M$ sends  to its neighbors a message $m_h^k:=\langle h,k,m, x_{m,h}^k,a_{m,h}^k,x_{m,h+1}^k,r_h^k\rangle$ containing step, episode, agent id, current state, current action, next state and current reward. Each neighbor then forwards the message to its neighbors. All  messages  older than $\gamma$ are excluded. Here  $0 \leq \gamma \leq D(G)$, where $D(\cdot)$ is graph diameter.
The message-passing protocol allows each agent to send information up to $\gamma$-hop neighbors. 
For $\gamma=0$, we recover  $M$ copies of the single-agent case with no communication and group regret $\mathcal O(M\sqrt{T})$. We make Assumption \ref{Assumption:message_life} in order to prevent propagation of unused messages in short episodes.
\begin{assumption} \label{Assumption:message_life} \normalfont{(Episodic length bounds message life).}
Assume $0 \leq \gamma \leq \min(D(G), H)$.   
\end{assumption}

We make several definitions for the $Q$-learning update. 

\begin{definition} \normalfont{(Number of state-action observations)} 
\begin{equation*}
    \begin{split}
        N_{m,h}^k(x,a) = \sum_{\ell=1}^k  \sum_{j=1}^M \mathds{1}[x_{j,h}^\ell = x, a_{j,h}^\ell = a] \mathds{1}[(m,j) \in E]
    \end{split}
\end{equation*}
\end{definition}
\noindent Counting the state-action visitations is crucial for controlling the error due to exploration in UCB-style algorithms. For any set $S$ of agents, we take $N_{S,h}^k(x,a)$ to be the smallest number of samples available to any agent $m \in S$.

Consider the power graph $G_\gamma$ of $G$: nodes $m$ and $m'$ share an edge if and only if $d(m,m') \leq \gamma$. Let $G_\gamma(m)$ be the neighbors of $m$ in $G_\gamma$. Define  set $\mathcal V_{m,h}^k(x,a)=\mathcal V$, 
\begin{align*}
    \mathcal V= 
    \begin{cases}
   \Big \{\underset{i \in G_\gamma(m)}{\bigcup}  (r_{i,h}^{k-1}, x_{i,h+1}^{k-1}) \Big \} & \exists i: N_{i,h}^k(x,a) > 0 \\ 
   \emptyset & \text{o.w.} 
\end{cases} 
\end{align*}
to be the set of all \textit{new} reward and next-state tuples available to each agent $m$ for state-action pair $(x,a)$ for any agent across the network at step $h$ and episode $k$. Define $\mathcal U_{m,h}^k$ to be the set of reward and next-state tuples taken by and observed by agent $m$. Then, $\mathcal U_{m,h}^k(x,a) =  \{(r_{m,h}^{k}, x_{m,h+1}^{k})\}$.   The update rule for $\mathcal V_{m,h}^k(x,a) $ is $$ \mathcal V_{m,h}^k = \mathcal U_{m,h}^k \cup \{\bigcup_{m' \neq m} \mathcal V_{m',h-d(m,m')}^k  \},$$
where $d(m,m')$ is the shortest-path distance between $m$ and $m'$ and $\mathcal V_{m,h^\prime}^k(x,a) = \emptyset$ if $N_{m,h^\prime}^k(x,a) = 0$ or $h^\prime \notin [h]$. 

Assume for fixed $(x,a,h,k)$, $|\mathcal V_{m,h}^k(x,a)| > 0$. Let $t = N_{m,h}^k(x,a)$. The $Q$-learning update is 
\begin{align*}
\begin{split}
    Q_{m,h}^{k+1}(x, a) :=& \sum_{(r,x^{\prime}) \in \mathcal V_{m,h}^k(x,a)} (1-\alpha_{m, t})Q_{m,h}^k(x, a) \\ &+\alpha_{m, t}[r + V_{m,h+1}^k(x^{\prime}) + b_{m,t}]
\end{split}
\end{align*}

The online multi-agent $Q$-learning algorithm with Hoeffding-style upper confidence bound is provided in Algorithm \ref{alg-ucb-hoeffding-with-comm}, where $\mathcal C(m)$ denotes the clique size of agent $m$ in the clique cover of $G_\gamma$.

\begin{algorithm}[h!]
\caption{Multi-Agent $Q$-learning with UCB-Hoeffding}
\label{alg-ucb-hoeffding-with-comm}
\begin{algorithmic}[1]
\INPUT $Q_{m,h}(x,a) \gets H$ and $N_{m,h}(x,a) \gets 0$ for all $(x,a,h,m) \in \mathcal S \times A \times [H]$
\For{episode $k=1,...,K$}
    \State receive vector $x_1$.
        \For{step $h=1,...,H$ and agent $m=1,...,M$ }
        \State Take action $a_{m,h}^k \gets \argmax_a'$
        $Q_{m,h}^k(x_{m,h}^k,a')$ \State $\quad$  and observe $x_{m,h+1}^k$.
        \State Update $\mathcal V_{m,h}^k$, $\mathcal U_{m,h}^k$ via message passing.
        \For{each $(x,a)$ with $(r,x^{\prime}) \in \mathcal V_{m,h}^k(x,a) \neq \emptyset$}
            \State $N_{m,h}^k(x,a) \gets N_{m,h}^k(x,a) + 1$
            \State $t \gets N_{m,h}^k(x,a)$
            \State $b_{m,t} \gets c \sqrt{H^3 \iota / (\mathcal C(m), t)}$
            \State $Q_{m,h}^k(x, a) \gets  (1-\alpha_{m, t})Q_{m,h}^k(x, a)$ \State $\quad + \alpha_{m, t}[r(x, a)+ V_{m,h+1}^k(x^{\prime}) + b_{m,t}]$
            \State $V_{m,h}^k(x_h) \gets H \wedge \max_{a' \in \mathcal A} Q_{m,h}^k(x_{m,h}^k,a') $
        \EndFor
    \EndFor
\EndFor
\end{algorithmic}
\end{algorithm}

Let $t= N_{m,h}^k(x,a)$ and suppose $(x,a)$ was previously observed at step $h$ of episodes $k_1,...,k_{t} < k,$ for each agent $m \in [M]$. In the following, let $(r,x^{\prime}) \in \mathcal V_{m,h}^{k}(x,a)$ have index $i \in [t]$. This makes the episode-wise $Q$-function determined as:
\begin{align*}
\begin{split}
    Q_{m,h}^k(x,a) &:=
     \alpha_{m,t}^0 H \\&+ \sum_{(r,x^{\prime}) \in \mathcal V_{m,h}^k(x,a)} \alpha_{m,t}^i [r_i + V_{m,h+1}^{k_i}(x_i^{\prime}) + b_{m,t}].\\
\end{split}
\end{align*}
As in 
\cite{jin_is_2018}, the $Q$-learning update is highly asynchronous. Thus, the regret bound depends on optimal choice of $\alpha_{m,t}$, $b_{m,t}$, network structure and $\gamma.$

\subsection{Convergence proof}

Let $(x,a)$ be a state-action pair, $m$ an arbitrary agent (estimating the value function and $Q$-function) and $(h,k)$ a step and episode index, respectively. We briefly state some notation from \cite{jin_is_2018} extended to multiple agents. 

\begin{definition} \normalfont{(Estimated policy performance gap)}
$$ \delta_{m,h}^k := (V_{m,h}^k - V_{m,h}^{\pi_{m,k}}) (x_{m,h}^k)$$
\end{definition}
\begin{definition} \normalfont{(Value estimation error)}
$$ \phi_{m,h}^k := (V_{m,h}^k - V_{m,h}^*) (x_{m,h}^k)$$
\end{definition}

\begin{definition} \normalfont{(Value gap due to modeling error)} \label{def:mod_error_martingale}
$$\xi_{m,h}^k := (\mathbb P_h - \hat{\mathbb P}_h)(V_{m,h+1}^* - V_{m,h+1}^{\pi_{m,k}}) (x_{m,h}^k, a_{m,h}^k)$$
\end{definition}
\noindent Note that $\xi_{m,h}^k$ is a martingale difference sequence. The key idea is to bound the estimated performance gap $\delta_{m,h}^k$ recursively for a given $h$ in terms of $\delta_{m,h+1}^k$, $\phi_{m,h+1}^k$, and $\xi_{m,h+1}^k$, following the procedure provided in \cite{jin_is_2018} (see Lemma \ref{lem:single_agent_val_error} and Theorem \ref{thm: ucb-hoeffding-with-comm}). The multi-agent communication strategy plays a critical role in choice of exploration. 

Lemmas 1 and 2 are reproduced from \cite{jin_is_2018}. Lemma 1 defines a learning rate that decays with the number of observations of a state-action pair, which is critical for optimal bootstrapping of the $Q$-function in the presence of exploration. Lemma 2 gives a bound for the approximation error for the $Q$-function.

\begin{lemma}
Let $\alpha_{m, t}:= \frac{H+1}{H+t}$, where $t$ denotes the number of times a state-action pair has been sampled, i.e. $t:=N_{m,h}^k(x,a)$. Let $\alpha_{m, t}^0 := \prod_{i=1}^t (1-\alpha_{m,i})$ and $\alpha_{m, t}^i := \alpha_{m,i} \prod_{j=i+1}^t (1-\alpha_{m,j}).$ The following  hold for $\alpha_{m, t}^i$:
\begin{enumerate}[label=(\alph*)]
    \item $\frac{1}{\sqrt{t}} \leq \sum_{i=1}^t \frac{\alpha_{m, t}^i}{\sqrt{i}} \leq \frac{2}{\sqrt{t}}$.
    \item $\max_{i \in [t]} \alpha_{m, t}^i \leq \frac{2H}{t}$ and $\sum_{i=1}^t (\alpha_{m, t}^i)^2 \leq \frac{2H}{t}$ for every $t \geq 1$. 
    \item $\sum_{t=i}^\infty \alpha_{m, t}^i= 1 + \frac{1}{H}$ for every $i \geq 1$. 
\end{enumerate}
\end{lemma}

\begin{lemma} \label{lemma:recursion}
\normalfont{(Recursion on $Q$)} For any $(m,x,a,h) \in  \mathcal S \times \mathcal A \times [H]$,  episode $k \in [K]$, let $t=N_{m,h}^k(x,a)$, and $(x,a)$ be observed by agent $m$ at episodes $k_1, ..., k_t < k$. Then
\begin{align*}
    \begin{split}
        & (Q_h^k - Q_h^*)(x,a) = \alpha_{m, t}^0(H - Q_h^*(x,a)) \\
        & + \sum_{i=1}^t \alpha_{m, t}^i (V_{m,h+1}^{k_i} - V_{m,h+1}^*)(x_{m,h}^k)\\
        & + \sum_{i=1}^t \alpha_{m, t}^i[(\hat{\mathbb P}_{h}^{k_{i}} - \mathbb P_h) V_{m,h+1}^*](x_{m,h}^k),a_{m,h}^k)) + b_{m,t}.
    \end{split}
\end{align*}
\end{lemma} 

Lemmas \ref{lem:xi_accum} and \ref{lem:single_agent_val_error} below are used in the subsequent proof of the regret bound. Lemma \ref{lem:xi_accum} uses a straightforward application of the Azuma-Hoeffding inequality to bound the summations of $\xi_{m,h}^k$ over $m$, $h$, and $k$. Lemma \ref{lem:single_agent_val_error} provides a single-agent bound on the estimated performance gap $\delta_{m,h}^k$ evaluated at time $h=1$, which is used in the group regret bound later. 

\begin{lemma} \normalfont{(Bound on $\xi_{m,h}^k$ accumulation)}
\label{lem:xi_accum} 
Let $\alpha \in (0,1)$ be a small failure probability. Then
$$ \sum_{m=1}^M \sum_{k=1}^K  \sum_{h=1}^H \xi_{m,h}^k \leq \sqrt{2H^3 MT \log \Big(\frac{2}{\alpha}}\Big) $$ 
with probability $1-\alpha$. 
\end{lemma}
\begin{proof} See Eqn. 41 of \cite{dubey_provably_2021}.
\end{proof}

\begin{lemma} \label{lem:single_agent_val_error} (Single agent cumulative value error). For fixed $h$, we have
\begin{align*}
    \begin{split}
        & \sum_{k=1}^K \delta_{m,1}^K \leq \mathcal O \Big (H^2SA + \sum_{h=1}^H \sum_{k=1}^K (e_{n_{m,h}^k} + \xi_{m,h+1}^k) \Big ). 
    \end{split}
\end{align*}
\end{lemma}
\begin{proof}
See Eqn. 4.8 of \cite{jin_is_2018}. 
\end{proof}

We now present the main contributions of the paper, Lemma \ref{lem:clique_bound} and Theorem \ref{thm: ucb-hoeffding-with-comm}, which utilize extra samples from the message passing introduced in Section \ref{sec:Algorithm} to provide sharper error bounds than those presented in \cite{jin_is_2018}.


\begin{lemma} \label{lem:clique_bound} \normalfont{(Clique bound on $Q$-error)} 
Let $p \in (0,1)$ be a small failure probability. Let $\mathbf C_\gamma$ be a partition of $G_\gamma$. There exists an absolute constant 
$c>0$ such that, if for every clique $\mathcal C \in \mathbf C_\gamma$ we assign exploration bonus $b_{m,t} = b_{\mathcal C,t} = c\sqrt{H^3 \iota / (|\mathcal C|t)} \; \forall m \in \mathcal C$, we have bounded error rate $e_{m, t} = e_{\mathcal C, t} = 2\sum_{m\in \mathcal C}\sum_{i=1}^{t_m} \alpha_{1, t}^i b_i \leq 4c\sqrt{|\mathcal C| H^3 \iota / t}$, $\forall m \in \mathcal C$, where $t_m = N_{m, h}^k(x,a)$ is the number of times $(x,a)$ has been observed by step $h$ in episode $k$ by agent $m$. Further, with probability at least $1-p$, we have simultaneously for all $(x,a,h,k,m) \in \mathcal S \times \mathcal A \times [H] \times [K] \times [M]$:
\begin{align*}
    \begin{split}
        0 \leq &\sum_{m \in \mathcal C}(Q_{m,h}^k - Q_{m,h}^*)(x,a) \leq |\mathcal C|\alpha_{1, t}^0 H \\
        & + \sum_{m \in \mathcal C} \sum_{i=1}^{t_m} (V_{m,h+1}^{k_i}-V_{m,h+1}^*)(x_{m,h+1}^{k_i}) + e_{\mathcal C, t}.
    \end{split}
\end{align*}
\end{lemma}
\noindent \begin{proof}  First, consider a clique $\mathcal C \in \mathbf C$. Assume $\gamma > 0$. Fix $(x,a,h) \in \mathcal S \times \mathcal A \times [H]$.
Let $k_i$ be the episode where $(x,a)$ is seen for the $i$th time (by any agent), and $k_i=K+1$ if $(x,a)$ hasn't been seen for the $i$th time yet. Thus, $k_i$ is a random variable and a stopping time. Next, let
 $X_{m,i} = \mathds{1}[k_i \leq K] \cdot [(\hat{\mathbb P}_{h}^{k_i} - \mathbb P_{h})V_{m,h+1}^*](x,a)$.
For filtration $\{\mathcal F_i\}_{i\geq0}$, it can be seen that $\{X_i\}_{i\geq0}$ is a martingale difference sequence by taking $\mathbb E[X_{m,i} | \mathcal F_{i-1}] \mathds{1}[k_i \leq K] \cdot \mathbb E[(\hat{\mathbb P}_h^{k_i} - \mathbb P_h)V_{m,h+1}^* | \mathcal F_{i-1}] 
    = \mathds{1}[k_i \leq K] \cdot \mathbb E[V_{m,h+1}^*(x_{h+1}^{k_i}) 
     \quad -  \mathbb E_{x^{\prime} \in \mathbb P_h(\cdot | x,a)} V_{m,h+1}^*(x^{\prime}) | \mathcal F_{i-1}]= 0$.

Let $\tau_m$ represent the fixed total number of $Q$-updates made by each agent $m$, and $\tau_{\mathcal C}$ the number of samples generated by just the clique. We proceed to bound the augmented trajectory length $\tau_m$ by the total samples generated by the clique. Consider an index set $I_{\mathcal C} \subseteq [\tau_m]$ for an agent $m \in \mathcal C$, which represents the total number of samples generated by the clique, where $|I_{\mathcal C}| = \tau_{\mathcal C}$ since an agent's clique is connected under $G_\gamma$. Let $I_{G\backslash \mathcal C} := [\tau_m] \backslash I_{\mathcal C}$, i.e. $|I_{G \backslash\mathcal C}| = \tau_M - |I_{\mathcal C}|$. By the Azuma-Hoeffding inequality for fixed $(x,a,h)$, we have for any collection of \textit{fixed} data-sets $\{\tau_m\}_{i=1}^M$ of size $\tau_m \in [0,MK] \; \forall m$, with probability $1-p/(MSAH)$:
\begin{align}
\label{eqn:first_hoeffding}
    \begin{split}
        &\Bigg|  \sum_{m \in \mathcal C} \sum_{i=1}^{\tau_m} \alpha_{m,\tau_{m}}^i \cdot  X_{m,i} \Bigg| \\
        &\leq \Bigg|  \sum_{m \in \mathcal C} \sum_{i\in I_{\mathcal C}} \alpha_{m,\tau_{m}}^i \cdot  X_{m,i} \Bigg| + \Bigg|  \sum_{m \in \mathcal C} \sum_{i\in I_{G\backslash\mathcal C}} \alpha_{m,\tau_{m}}^i \cdot  X_{m,i} \Bigg|\\
        &\leq \Bigg|  \sum_{m \in \mathcal C} \sum_{i=1}^{\tau_{\mathcal C}} \alpha_{m,\tau_{\mathcal C}}^i \cdot     X_{m,i} \Bigg|  + \Bigg|  \sum_{m \in \mathcal C} \sum_{i=1}^{\tau_{m}-\tau_{\mathcal C}} \alpha_{m,\tau_m - \tau_{\mathcal C}}^i  X_{m,i} \Bigg| \\
        &\leq \sqrt{\frac{|\mathcal C|H^3 \iota}{\tau_{\mathcal C}}} +  \sqrt{\frac{|\mathcal C|H^3 \iota}{\tau_{m}-\tau_{\mathcal C}}}, \\
    \end{split}
\end{align}
where the first line is due to the triangle inequality, and the second line follows since $\{\alpha_{m,M}^i\}_{i=1}^M$ is dominated pointwisely by $\{\alpha_{m,N}^i\}_{i=1}^N$ for any $M>N$. Since $\tau \in [0, |\mathcal C|\cdot K]$, the above bound also holds for $\tau_{\mathcal C} = t := N_{m,h}^k (x,a)$, which is a random variable. Hence, Eqn.~\ref{eqn:first_hoeffding} shows that the sum of the martingale error decomposes into two terms that decay with the number of in-clique and out-of-clique observations. Since $\mathds{1}[k_i \leq K]$ for any $i \leq N_{m,h}^k(x,a)$, we have (via a union bound) that with probability $1-p$, for all $(m,x,a,h,k) \in [M] \times \mathcal S \times \mathcal A \times [H] \times [K]$, $|\sum_{m \in \mathcal C} \sum_{i=1}^{t } \alpha_{\tau_{m}}^i \cdot X_{m,i} | \leq  c \sqrt{|\mathcal C|\frac{H^3 \iota}{t}}$. 
The rest of the proof follows from \cite{jin_is_2018}, Lemma 4.3: picking exploration bonus $b_{m, t} := c \sqrt{\frac{H^3 \iota}{t|\mathcal C|}}$ and plugging into Lemma \ref{lemma:recursion}, a bound on the the clique $Q$ error is  

\begin{align}
\label{eqn:clique_Q_error}
    \begin{split}
        0 \leq& \sum_{m \in \mathcal C} (Q_{m,h}^k - Q_{m,h}^*)(x,a)\\
        \leq& \sum_{m \in \mathcal C}\alpha_{ t}^0H +\sum_{i=1}^{t_m} \alpha_{m,t}^i (V_{m,h+1}^{k_i} - V_{m,h+1}^*)(x_{h+1}^{k_i})+ e_{\mathcal C, t} \\
    \end{split}
\end{align}
where $\alpha_{t}^0$ denotes the default initial learning rate. Thus, exploration error accumulation $e_{m,t} = e_{\mathcal C,t} \leq 4c \sqrt{|\mathcal C|H^3 \iota/t}$, $\forall m \in \mathcal C$. 

\end{proof}

\begin{remarks}
This analysis allows agents to learn off-policy through messages containing the exploration-enriched experiences of other agents. That is, all agents in a clique $\mathcal C$ explore proportionally to the size of the clique, and share the resulting samples. To simplify the analysis, we assume that $G_\gamma$ is connected, i.e. all cliques communicate with other cliques and $d_{\mathcal C}^* > |\mathcal C|$. If this is not the case, then the second term in \eqref{eqn:first_hoeffding} drops out. \eqref{eqn:clique_Q_error} shows that the number of samples generated by clique $\mathcal C$ is sufficient to significantly reduce exploration error in $\mathcal C$, and the final regret bound depends on the high-order clique error term $e_{\mathcal C,t}$.
\end{remarks}

\begin{theorem}  \label{thm: ucb-hoeffding-with-comm} \normalfont{(Hoeffding regret bound for parallel MDP with communication)}
Let $p \in (0,1)$ be a small failure probability. There exists an absolute constant $c > 0$ such that if we choose $b_{m,t} = b_{\mathcal C,t} = c \sqrt{H^3 \iota / (|\mathcal C| t)} \; \forall m \in \mathcal C$, then with probability $1-p$, the group regret of 
multi-agent $Q$-learning with Algorithm \ref{alg-ucb-hoeffding-with-comm} is at most $\mathcal O(\sqrt{MH^4SAT\iota})$, where $\iota := \log(SATM/p)$.
\end{theorem}

\begin{proof}
Similar to \cite{jin_is_2018}, the group regret defined in \eqref{eqn:group_regret} is expanded as:
\begin{align}
\label{eqn:group_regret_expanded}
    \begin{split}
        \sum_{m=1}^M \text{Regret}_m(K)  &\leq \sum_{m=1}^M \sum_{k=1}^K \delta_{m,1}^k
    \end{split}
\end{align}

From the single-agent analysis \cite{jin_is_2018}, the regret at each step $h$ and episode $k$ is bounded recursively as 
\begin{align}
\label{eqn:recursive_error_demcomp}
    \begin{split}
        \delta_{m,h}^k &\leq (V_h^k - V_h^{\pi_{m,k}}) (x_{m,h}^k) \alpha_{m,t}^0 H +\sum_{i=1}^t \Big\{ \alpha_{m,t}^i\phi_{m,h+1}^{k_i}  \\
        &\quad + e_{m,t} \Big \} - \phi_{m,h+1}^k + \delta_{m,h+1}^k + \xi_{m,h+1}^k \\
    \end{split}
\end{align}
for $t:=N_{m,h}^k(x_{m,h}^k, a_{m,h}^k)$; see \cite{jin_is_2018} for the complete derivation.

Let $\mathbf C_\gamma$ denote a clique covering of $G_\gamma$. Lemma \ref{lem:single_agent_val_error} gives the solution to the recurrence in \eqref{eqn:recursive_error_demcomp} for each agent $m$. Following \cite{jin_is_2018} (Eqn. 4.7), we can bound the value gap as
\begin{align*}
        &\sum_{m=1}^M \sum_{k=1}^K \delta_{m,h}^k \leq \sum_{\mathcal C \in \mathbf C_\gamma} \sum_{m \in \mathcal C} \sum_{k=1}^K \delta_{m,h}^k \\
        & \leq 
        \sum_{\mathcal C \in \mathbf C_\gamma} \sum_{m \in \mathcal C} \mathcal O \Big (HSA + \sum_{k=1}^K \Big\{ e_{m,n_{m,h}^k} +\xi_{m,h+1}^k \Big\}\Big ) \\
        & \leq 
        \mathcal O(1) \sum_{\mathcal C \in \mathbf C_\gamma} \sum_{m \in \mathcal C} \sum_{k=1}^K e_{m,n_{m,h}^k(x_{m,h}^k, a_{m,h}^k)} 
\end{align*}
where the second inequality 
is from Lemma \ref{lemma:recursion} and the third 
from Lemma \ref{lem:xi_accum}. The total group regret is bounded as
\begin{align}
\label{eqn:final_regret}
    \begin{split}
        &\sum_{m=1}^M \text{Regret}_m(K) \leq \sum_{\mathcal C \in \mathbf C_\gamma} \sum_{m \in \mathcal C} \sum_{k=1}^K \delta_{m,1}^k \\
        &\leq 
        \mathcal O(1)  \sum_{\mathcal C \in \mathbf C_\gamma} \Bigg\{ \sqrt{\frac{KH^3 \iota}{d_{\mathcal C}^* - |\mathcal C|}}  +\sum_{k=1}^K |\mathcal C| \sup_{\substack{(m,x,a): \\ \mathcal U_{m,h}^k \neq \emptyset}}\sqrt{\frac{|\mathcal C|H^3\iota}{N_{\mathcal C,h}^{k}(x,a)}} \Bigg\}
        \\   
        & \leq 
        \mathcal O(1) \sqrt{\frac{\bar \chi (G_\gamma) KH^3 \iota}{d_{\mathbf G_\gamma}^\text{avg}}} + \mathcal O(1) \sum_{\mathcal C \in \mathbf C_\gamma}  \sqrt{|\mathcal C|H^3SAK\iota } \\
        &= \mathcal O(1) \sqrt{\frac{\bar \chi (G_\gamma) KH^3 \iota}{d_{\mathbf G_\gamma}^\text{avg}}} + \mathcal O(1)   \sqrt{\bar \chi (G_\gamma)MH^2SAT\iota }\\
        &= \Tilde{\mathcal O} \Bigg (\sqrt{\frac{\bar \chi (G_\gamma) H^4T }{d_{\mathbf G_\gamma}^\text{avg}}} +  \sqrt{\bar \chi (G_\gamma)MH^4SAT} + \sqrt{H^3MT} \Bigg )
    \end{split}
\end{align}
with probability $1-p$, where the second inequality 
comes from the fact that for any  $(x,a,k,h)$, the number of external messages $ N_{\mathcal C,h}^{G\backslash \mathcal C, k}(x, a) \leq k|\mathcal  C|(d_{\mathcal C}^* - |\mathcal C|)$, and $d_{\mathcal C}^*$ is the max degree in the clique. Intuitively, each agent in the clique receives no more than $d_{\mathcal C}^* - |\mathcal C|$ out-of-clique samples corresponding to step $h$ per episode. The third inequality 
is due to  Cauchy-Schwartz, where $d_{\mathbf G_\gamma}^\text{avg} := (\sum_{C \in \mathbf C_\gamma} (d_{\mathcal C}^* - |\mathcal C|)^{-1} )^{-1}$ is a constant that represents the average cross-clique degree structure, and $\bar \chi (G_\gamma)$ denotes the clique covering number (i.e., minimum number of cliques in a clique covering) of the communication power graph $G_\gamma$. In the worst case, when exploration is uniformly random, the number of clique samples $N_{\mathcal C, h}^{k}(x,a) = |\mathcal C|^2k/SA$. Note that if the initial joint state is kept fixed, \eqref{eqn:final_regret} implies a per-agent sample complexity of $\mathcal O(\varepsilon^{-2}M^{-1/2})$ (see \cite{jin_is_2018} for further details). 
\end{proof}

\begin{remarks} (On trade-off between network size and complexity).  \eqref{eqn:final_regret} bounds the regret as a trade-off between the number of agents in the network $G$ and the clique covering number communication power graph $G_\gamma$. Since $\sqrt{H^3MT} \leq \sqrt{MH^4SAT}$, we reasonably can neglect the third term of \eqref{eqn:final_regret} and focus on the first two terms. When $d_{\mathbf G_\gamma}^\text{avg} SA M \geq 1$, the second term dominates. For example, a line graph, which has a higher clique-to-agent ratio, has $d_{\mathbf G_\gamma}^\text{avg} \geq \mathcal O(1/M)$, which permits the second term to be higher-order regardless of $S$ and $A$. Also, we note that in many practical applications, $SA >> M$. Finally, note that even for a disconnected graph, \eqref{eqn:final_regret} still holds because $\bar \chi (G_\gamma) \leq M $ (the inequality is tight for \textit{completely} disconnected graph), and the first term will drop out (see Lemma  \ref{lem:clique_bound}). Furthermore, for a completely disconnected graph, we recover the naïve parallel regret bound of $\Tilde{\mathcal O}(M\sqrt{H^4SAT})$ presented earlier. 
\end{remarks}

\begin{remarks} (Comparison with centralized regret).
Our regret bound matches the group regret bound $\mathcal O(\sqrt{MT})$ for a fully centralized agent running $MK$ episodes \cite{dubey_provably_2021}.
\end{remarks}

\section{Experimental Results} \label{simulations}

In this section we provide numerical simulations to illustrate results and validate theoretical bounds. We consider a cooperative game where the goal of each agent is navigating to a specific landmark. This is a modified version of Cooperative Navigation task \cite{lowe_multi-agent_2020-1,terry_pettingzoo_2021}.


Here we consider $M$ agents and $M$ landmarks. Agents are modeled with double-integrator dynamics; the (continuous) state space is $4$-dimensional and consists of a planar position and a planar velocity. In our modified environment, agents are assigned a landmark and rewarded based on their distance to the landmark. We discretize the continuous state space into a $10 \times 10 \times 10 \times 10$ grid on the box $\mathcal S = [-2,2]^4$, giving $10^4$ total states. The agents' are initialized to lie within $[-1,1]^4$; states outside of $\mathcal S$ are mapped to the nearest state in $S$. The action space is discrete and consists of movements {\tt{left, right, down, up}} and {\tt{no-op}}. The reward assigned to each agent is given as the Euclidean distance to its assigned landmark. Since each agent is rewarded locally, i.e., independent of interaction with other agents,  
the optimal joint policy is attained through convergence to the optimal local policy for each agent. 

We highlight how communication accelerates convergence to the optimal local policy, including when agents must operate in under-explored regions of the state space. Training is performed according to Algorithm \ref{alg-ucb-hoeffding-with-comm}, and test performance is given as the average reward accrued over an episode of length $H=10$ under the implied policy $\pi_{m,h}^k(s) := \argmax_{a \in \mathcal A} Q_{m,h}^k(s, a)$ for any agent $m$, state $s$, step $h$, and episode $k$. Plots in Figure~\ref{fig:message_comm_perf_plots} show average performance over $10$ trials.

\subsection{Message Passing Communication with 4 Agents}


In this scenario, $M=4$ agents (A, B, C, and D) must navigate to assigned landmarks that are roughly at the same location (also at the origin). During training, agents are assigned initial positions: A with $[-1, 1]/\sqrt{2}$, B with $[-1, -1]/\sqrt{2}$, C with $[1, -1]/\sqrt{2}$, and D with $[1, 1]/\sqrt{2}$. The communication network is an undirected line graph with edges between A and B, B and C, and C and D. We take the message life parameter to be the diameter of the graph, i.e. $\gamma=3$.  At test time the positions of A and D are switched. Since A and D do not share an edge, they must communicate state, action and reward samples through \textit{message passing}. Since each episode is $H=10 > \gamma$, samples are fully propagated through the network one episode (i.e. $H=10$ steps) later. If $\gamma=3$, A and D show improvement when their initial condition is swapped. Fig.~\ref{fig:message_comm_perf_plots} plots the average reward per episode versus the training episode. 
\begin{figure}[ht]
    \centering
    \includegraphics[width=0.3\textwidth]{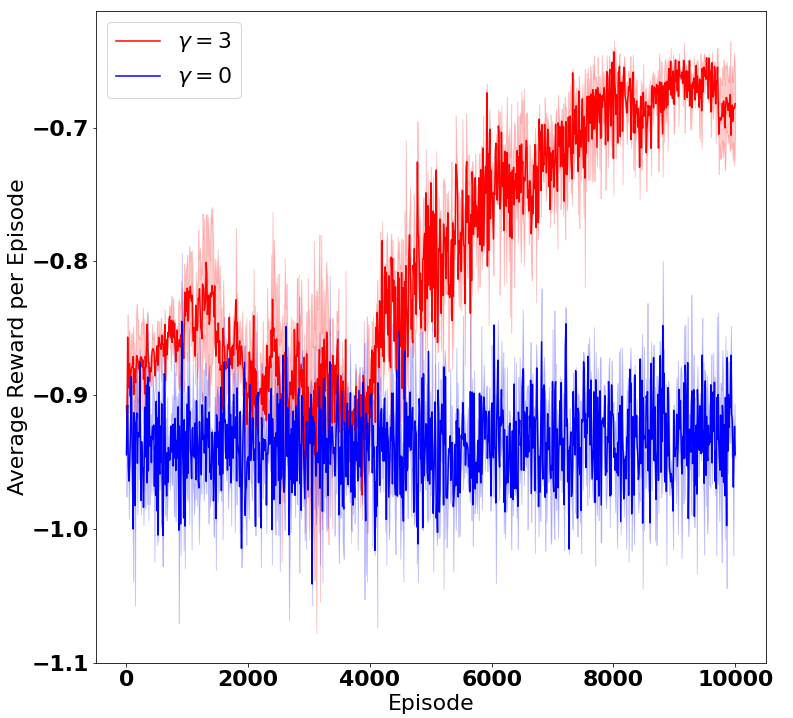}
    \caption{Average episode reward achieved per episode for the 4-agent scenario for different $\gamma$. The reward is averaged over the A and D agents. When the initial states of A and D, which do not share an edge, are swapped, both are able to learn off-policy from message passing data.}
    \label{fig:message_comm_perf_plots}
\end{figure}

\section{Conclusion} \label{conclusion}

Optimal exploration in online reinforcement learning is a key consideration that impacts  sample complexity. We investigate the benefits of fully decentralized exploration in MARL using regret as a metric. We provide a multi-agent extension of the UCB-Hoeffding algorithm provided in \cite{jin_is_2018} where the agents are equipped with a message passing scheme. We prove a regret bound that is $\Tilde{\mathcal O} (\sqrt{MH^4SAT})$, an $\Tilde{\mathcal O} (M^{-1/2})$ improvement over the single-agent setting  Specifically, we consider general network $G$ and show that the regret also depends on the clique structure of the power graph $G_\gamma$, in addition to the number of agents. This key result suggests that the dense network structure, higher message life $\gamma$, and higher number of agents $M$ all reduce the average regret incurred by each agent, as has been shown in simulation. While the assumption of time-varying dynamics demands samples that are polynomial in the episode length $H$, cooperative estimation of the optimal value functions allow parallel experience generation and off-policy learning using communication. When the initial state is fixed, our regret bound corresponds to $\mathcal O(\varepsilon^{-2}M^{-1/2})$ sample complexity. Further work may involve providing a multi-agent lower bound and minimax-optimal regret upper bound using a Bernstein-style or similar UCB bonus as shown in \cite{jin_is_2018, zhang_almost_2020-1}. Further, the message life $\gamma$ should be optimized with respect to the number of agents and network structure by considering communication cost and privacy. We plan to extend our tabular results to deep $Q$-learning.







\bibliographystyle{IEEEtran}
\bibliography{root}

\end{document}